\documentclass{ecai}

\usepackage{amsmath}
\usepackage{amssymb}
\usepackage{amsthm}
\usepackage{booktabs}
\usepackage{graphicx}
\usepackage{color}
\usepackage{pifont}
\usepackage[utf8]{inputenc} 
\usepackage{booktabs}       
\usepackage{amsfonts}       
\usepackage{nicefrac}       
\usepackage{xcolor}         
\usepackage{subcaption}
\usepackage{tikz}
\usetikzlibrary{calc, positioning, shadows}
\usepackage{pgfplots}
\pgfplotsset{compat=1.18}

\usepackage{multirow}
\usepackage{algorithm}
\usepackage{algorithmic}
\usepackage{xspace}

\usepackage{microtype}
\usepackage{graphicx}
\usepackage{booktabs} 

\usepackage[pagebackref,breaklinks,colorlinks,allcolors=blue]{hyperref}
\usepackage{cleveref}

\newcommand{\rex}{\textsc{r}e\textsc{x}\xspace}

\newcommand{\gradcam}{\textsc{g}rad\textsc{c}am\xspace}
\newcommand{\lime}{\textsc{lime}\xspace}
\newcommand{\shap}{\textsc{shap}\xspace}

\newcommand{\multideepcover}{\textsc{MultEX}\xspace}
\newcommand{\ie}{\emph{i.e.}\xspace}

\newcommand{\sag}{\textsc{Sag}\xspace}

\newtheorem{theorem}{Theorem}
\newtheorem{lemma}[theorem]{Lemma}

\newtheorem{proposition}[theorem]{Proposition}

\newtheorem{definition}{Definition}

\newcommand{\lem}{\begin{lemma}}
\newcommand{\elem}{\end{lemma}}
\newcommand{\pro}{\begin{proposition}}
\newcommand{\epro}{\end{proposition}}
\newcommand{\dfn}{\begin{definition}}
\newcommand{\edfn}{\end{definition}}
\newcommand{\K}{{\mathcal{K}}}

\newcommand{\Scal}{{\cal S}}

\newcommand{\bd}{\begin{definition}}
\newcommand{\ed}{\end{definition}}
\newcommand{\be}{\begin{enumerate}}
\newcommand{\bi}{\begin{itemize}}
\newcommand{\ee}{\end{enumerate}}
\newcommand{\ei}{\end{itemize}}

\newcommand{\U}{{\cal U}}

\newcommand{\cF}{{\cal F}}
\newcommand{\V}{{\cal V}}
\newcommand{\R}{{\cal R}}

\newcommand{\stam}[1]{}

\newcommand{\commentout}[1]{}

\newcommand{\clm}{\begin{claim}}
\newcommand{\eclm}{\end{claim}}

\newcommand{\thm}{\begin{theorem}}
\newcommand{\ethm}{\end{theorem}}

\definecolor{darkred}{rgb}{0.65,0,0}

\newcommand{\sat}{\models}

\newcommand{\xam}{\begin{example}}
\newcommand{\exam}{\end{example}}

\begin{document}

\begin{frontmatter}
\paperid{7738} 

\title{Multiple Different Black Box Explanations \\ for Image Classifiers}
\author[A]{Hana Chockler}
\address[]{}
\author[A]{David A. Kelly}
\address[]{}
\author[B]{Daniel Kroening}

\address[A]{King's College London, UK}
\address[B]{Amazon, UK}

\begin{abstract}
  Existing explanation tools for image classifiers usually give only
  a single explanation for an image's classification. 
  For many images, however, image
  classifiers accept more than one explanation for the image label.
  These explanations are useful for analyzing the decision process of the classifier and
  for detecting errors.
  Thus, restricting the number of explanations to just one severely limits
  insight into the behavior of the classifier.
  In this paper, we describe an algorithm and a tool,
  \multideepcover, for computing multiple explanations as the output of a
  black-box image classifier for a given image. Our algorithm uses a principled
  approach based on actual causality. We analyze its theoretical complexity and evaluate
  \multideepcover against the state-of-the-art across three different models and three different datasets.
  We find that \multideepcover finds more explanations and that these explanations are of higher quality.
\end{abstract}

\end{frontmatter}

\section{Introduction}%
\label{sec:introduction}

\begin{figure*}[t]
    \centering 
    \begin{subfigure}[b]{0.3\textwidth}
        \centering
        \includegraphics[height=100pt, width=100pt]{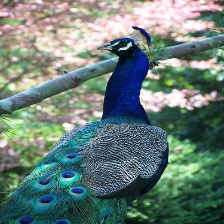}
        \caption{Class $84$: peacock}\label{subfig:peacock}
    \end{subfigure}
    \vspace{5mm}
    \begin{subfigure}[b]{0.3\textwidth}
        \centering
        \includegraphics[height=110pt, width=110pt]{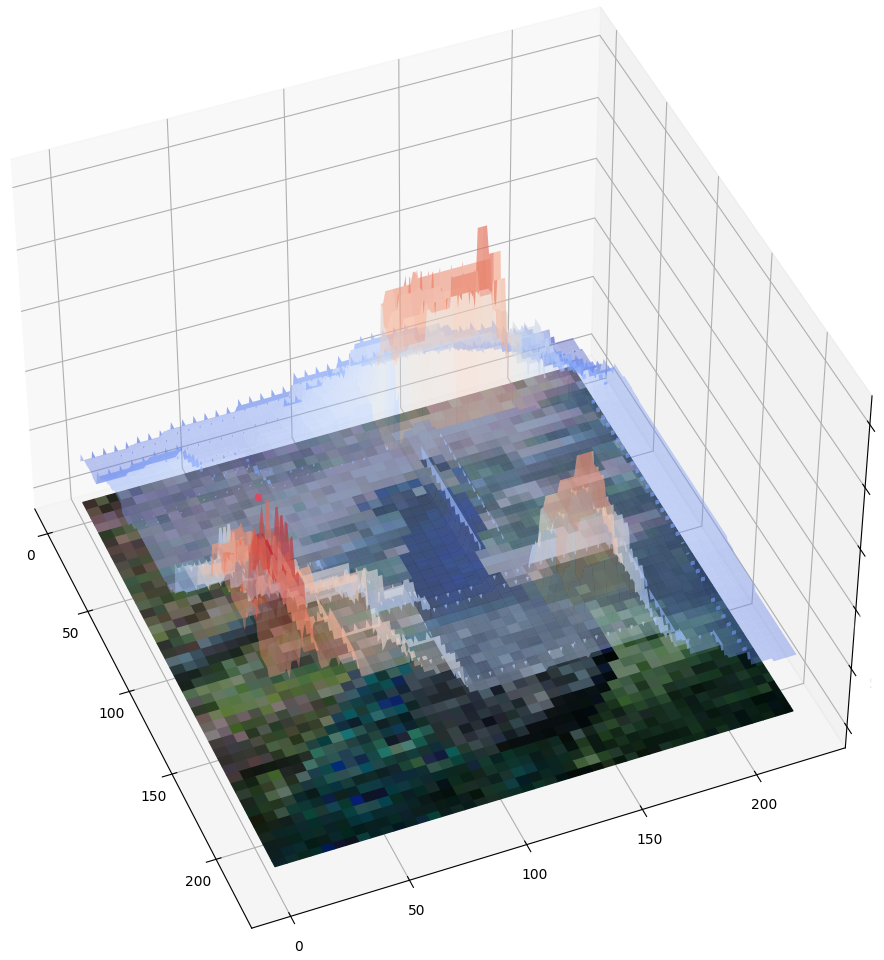}
        \caption{Responsibility Map}\label{subfig:peacock_rm}
    \end{subfigure}
    \begin{subfigure}[b]{0.3\textwidth}
        \centering
        \includegraphics[height=100pt, width=100pt]{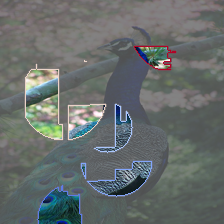}
        \caption{Explanations with no overlap}\label{subfig:peacock_no_over}
    \end{subfigure}
    \begin{subfigure}[b]{0.3\textwidth}
        \centering
        \includegraphics[height=100pt, width=100pt]{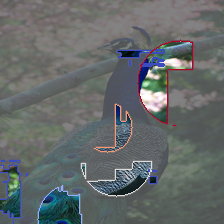}
        \caption{Explanations with partial overlap}\label{subfig:peacock_03_over}
    \end{subfigure}
    \begin{subfigure}[b]{0.3\textwidth}
        \centering
        \includegraphics[height=100pt, width=100pt]{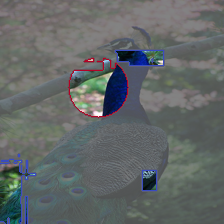}
        \caption{Explanations with high confidence}\label{subfig:peacock_high_conf}
    \end{subfigure}
    \begin{subfigure}[b]{0.3\textwidth}
        \centering
        \includegraphics[height=100pt, width=100pt]{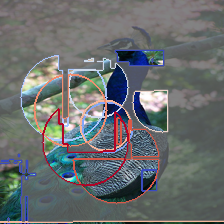}
        \caption{High confidence with overlap}\label{subfig:peacock_high_over}
    \end{subfigure}
    \vspace{5mm}
    \caption{\multideepcover on a peacock. \multideepcover computes a responsibility landscape (\Cref{subfig:peacock_rm}): this landscape encodes so much information about the image that, from it, we can compute non-overlapping explanations (\Cref{subfig:peacock_no_over}, partially overlapping to any degree (\Cref{subfig:peacock_03_over}), explanations which have higher confidence than the original image (\Cref{subfig:peacock_high_conf} and explanations with higher confidence than the original image, and total permissible overlap (\Cref{subfig:peacock_high_over}).}\label{fig:accept}
    \vspace{2mm}
\end{figure*}

AI models are now a primary building block of most computer vision systems. The opacity of
some of these models (e.g., neural networks) creates demand for explainability techniques, 
which attempt to provide insight into
why a particular input yields a particular observed output.
Beyond increasing a user's confidence in the output, and hence also their trust
in the AI model, these insights help to uncover subtle classification errors
that are not detectable from the output alone~\cite{chockler2024}.





Existing explainability tools use a variety of definitions for explanations, often
tied to a particular method of extracting them.
The definition we use here is grounded in the theory of actual causality and, roughly
speaking, defines an explanation as a smallest part of the input image  
that is sufficient for the classifier to yield the same top label as the original
image (see~\Cref{sec:cause}). If the top classification for the model is, say, `peacock' (\Cref{fig:accept}), then a causal explanation is a subset of image pixels also labeled `peacock' as the top classification.
Explanations need not be unique. A human could point to many parts of the image of a peacock and state that this part is sufficient to label the entire image `peacock'. There is not reason, a priori, to assume that a model cannot do the same. 





Why do we need to have multiple explanations? Previous work, using a user study, demonstrating that people value having multiple explanations of the model's decisions, as it increases their confidence in the classification and gives them more insight into the reasoning process of the model~\cite{SLKTF21}.
The increase in confidence and trust is reaffirmed in~\cite{Han10,Mil19}.
Even, perhaps, more importantly, as we show in this paper, the prevalence of multiple explanations suggests that algorithms for computing more than one explanation are essential for understanding the reasoning process of image classifiers and uncovering subtle classification errors. 

To illustrate the latter point, the image in \Cref{fig:misclassification} is classified by a ResNet50 as `tennis racket'.~\Cref{subfig:6,subfig:9} both show that sections of the image that overlap with the racket are sufficient for the overall classification.~\Cref{subfig:1}, however, shows a 
part of the player's shorts as being sufficient for the `tennis racket' classification, with \emph{higher} confidence ($0.25$) that for~\Cref{subfig:6,subfig:9}, which are both around $0.23$. This is a concerning finding for the users of the model.

Most existing techniques provide only one explanation, potentially missing the error.
The one notable exception is \sag~\cite{SLKTF21}, a pioneering work in the exploration of multiple explanations for image classifiers. 
However, \sag suffers from a number of shortcomings. First, it is not a true black-box tool, as it requires access to the gradient of the model. 
Such tools might be more accurately called \emph{grey-box}, as a proprietary model may not expose the gradient during inference.
Furthermore, \sag{}'s concept of explanation is rather different from what a human might accept as an explanation, and is very different from a causal explanation. \sag{} is liberal in what it considers an explanation, resulting in potentially thousands of them found for a single image.~\cite{Jiang_2024_CVPR} discovered this when investigating ``compositionality'', which they define as a conjunction of parts of an image (patches) that have high likelihood ratios for a particular classification. Intuitively though, it is unlikely that an image classified as, say `dog', has thousands of explanations: it happens because \sag{}'s explanation refers to \emph{any} element in the output tensor above a predefined probability threshold. We argue that this renders \sag{}'s results uninformative. 
Indeed, with a suitable threshold, \emph{any} combination of patches is sufficient for the desired classification. 
Causality, on the other hand, says that a set of pixels explains label `a' if that set is sufficient to be the actual (top) classification of the model. This is far stricter than \sag{}'s explanations.

To overcome these problems, we present~\multideepcover, a black-box algorithm and a tool for computing multiple explanations for image classifiers. Using the theory of actual causality, \multideepcover computes a causal responsibility ranking of the pixels of the image, from which it extracts multiple different explanations. Unlike with \sag, \multideepcover is not fixed to a rigid grid, so its explanation discovery is much more flexible (\Cref{fig:accept}). \multideepcover also allows an optional confidence threshold: unlike with \sag, where the threshold dictates what constitutes an explanations, \multideepcover{}'s threshold affects the \emph{quality} of an explanation (\Cref{subfig:peacock_high_conf}), picking out subsets of pixels which can have even higher confidence than the entire image. Furthermore, the explanations produced by
\multideepcover are \emph{actual}, that is, they always explain the top (the most likely) classification of the input image. This in contrast
to \sag that produces explanations that do not necessarily correspond to the top classification (see \Cref{sec:experiments}).
In \Cref{sec:theory}, we also present an exponential upper bound on the number of possible explanations and
demonstrate that this bound is tight. In \Cref{sec:experiments} we experimentally compare \multideepcover with \sag
on standard benchmarks. We show that \multideepcover consistently finds more explanations than \sag.
We also show that \multideepcover performs well when we add in a confidence threshold to increase explanation quality. 

We provide the details of the benchmark sets, the models, and the main results in the paper. 
The tool, all datasets, and the full set of reported as well as additional results are submitted as a part of the supplementary material.

\section{Related Work}\label{sec:related}

There is a large body of work on algorithms for computing one explanation
for a given output of an image classifier. They can be largely grouped into white-box and black-box methods. 
White-box methods frequently use variations on propagation-based
explanation methods to back-propagate a model's decision to the input
layer to determine the weight of each input
feature for the decision~\cite{springenberg2015striving,
sundararajan2017axiomatic, bach2015pixel, shrikumar2017learning,
nam2020relative}. \gradcam, a white-box technique which has spawned many variants, only needs
one backward pass and propagates the class-specific gradient into the final convolutional layer of a 
DNN to coarsely highlight important regions of an input image~\cite{CAM}. 
%
%

Perturbation-based explanation approaches introduce perturbations to the input space directly
in search for an explanation. These are typically found in black-box explanation methods.
\shap (SHapley Additive exPlanations) 
computes Shapley values of different parts of the input and uses them to rank the
features of the input according to their importance~\cite{lundberg2017unified}.
\lime constructs a simple model to label the original input and its neighborhood of perturbed
images and uses this model to estimate the importance of different parts of the input
~\cite{lime, datta2016algorithmic,
chen2018learning, petsiuk2018rise, fong2019understanding}. 
Finally, \rex~\cite{chockler2024} ranks elements of the image according to their responsibility for the classification
and uses this ranking to greedily construct a small explanation. 
The \rex ranking procedure is based on an approximate computation of causal responsibility. 
None of these black-box tools provide multiple explanations of the same image.

Work on calculating more than one explanation for a given classification outcome
is in its infancy. 
A recent paper on abductive explanations raises the problem of generating only one explanation and suggests \emph{aggregating} 
multiple explanations to obtain the full information about the importance of different features
of the input~\cite{BILVZ24}. However, in the absence of a tool for reliably and systematically
generating multiple explanations, they rely on re-executing \lime and \shap multiple times, in the hope that
their inherent non-determinism results in different explanations. This is clearly not a systematic
approach.

To the best of our knowledge, there is only one algorithm and tool that specifically
computes multiple explanations of image classifiers---\sag, described in~\cite{SLKTF21}.
We describe \sag's algorithm in more detail in~\Cref{sec:experiments}
and argue that its definition of explanation is too permissive. Our experimental results show that, even with our stricter definition, we find more explanations than \sag.

Finally, we mention a growing body of work on logic-based explanations~\cite{INMS19,MSI22,DH23}, where a symbolic encoding
of the model is given. Their notion of abductive explanations is similar in spirit to the one used in this paper,
except all possible values of pixels outside of the explanations are considered. 
The problem setting is very different from ours, considering the model as a logic formula. 
In contrast, our approach is black box and is agnostic to the internal structure of the classifier, nor does it try to represent its decision process as a logic expression.

\section{Background on Actual Causality}\label{sec:cause} 

\begin{figure*}[t]
    \centering
    \begin{subfigure}{0.2\textwidth}
        \includegraphics[height=90pt, width=90pt]{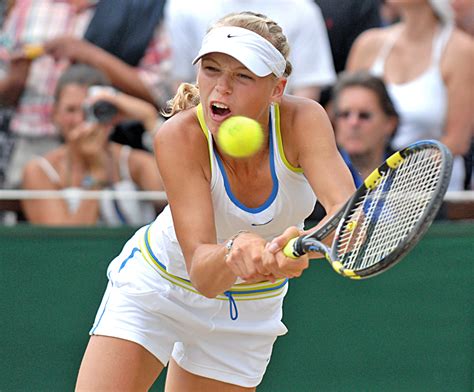}
        \caption{Class $752$: racket}
    \end{subfigure}
    \hfill
     \begin{subfigure}{0.2\textwidth}
        \includegraphics[height=90pt, width=90pt]{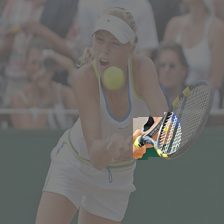}
        \caption{}\label{subfig:9}
    \end{subfigure}
    \hfill
    \begin{subfigure}{0.2\textwidth}
        \includegraphics[height=90pt, width=90pt]{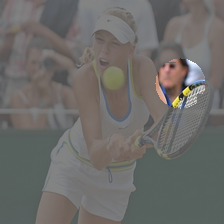}
        \caption{}\label{subfig:6}
    \end{subfigure}
    \hfill
    \begin{subfigure}{0.2\textwidth}
        \includegraphics[height=90pt, width=90pt]{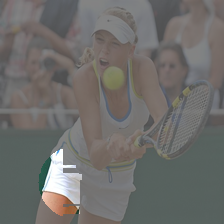}
        \caption{}\label{subfig:1}
    \end{subfigure}
        \hfill
    \vspace{3mm}
    \caption{Imagenet class $752$: racket, according to ResNet50.~\Cref{subfig:1,subfig:6,subfig:9} show $3$ minimal, sufficient explanations for class $752$.
    Only~\Cref{subfig:6,subfig:9} contains part of the racket. The tennis players shorts are also classified as racket, with a higher confidence than either~\Cref{subfig:9} or~\Cref{subfig:6}.}\label{fig:misclassification}
    \vspace{2mm}
\end{figure*}

In this section we briefly review the definitions of causality and causal
models introduced by Halpern and Pearl \cite{HP01b} and
relevant definitions of causes and explanations in image
classification~\cite{CH24}. The reader is referred
to~\cite{Hal19} for further reading.

We assume that the world is described in terms of 
variables and their values.  
Some variables may have a causal influence on others. This
influence is modeled by a set of {\em structural equations}.
It is conceptually useful to split the variables into two
sets: the {\em exogenous\/} variables, whose values are
determined by 
factors outside the model, and the {\em endogenous\/} variables, whose values are ultimately determined by
the exogenous variables.  
The structural equations describe how these values are 
determined.

Formally, a \emph{causal model} $M$
is a pair $(\Scal, \cF)$, where $\Scal$ is a \emph{signature}, which explicitly
lists the endogenous and exogenous variables  and characterizes
their possible values, and $\cF$ defines a set of \emph{(modifiable)
structural equations}, relating the values of the variables.  
A signature $\Scal$ is a tuple $(\U,\V,\R)$, where $\U$ is a set of
exogenous variables, $\V$ is a set 
of endogenous variables, and $\R$ associates with every variable $Y \in 
\U \cup \V$ a nonempty set $\R(Y)$ of possible values for 
$Y$ (i.e., the set of values over which $Y$ {\em ranges}).  
For simplicity, we assume here that $\V$ is finite, as is $\R(Y)$ for
every endogenous variable $Y \in \V$.
$\cF$ associates with each endogenous variable $X \in \V$ a
function denoted $F_X$
(i.e., $F_X = \cF(X)$)
such that $F_X: (\times_{U \in \U} \R(U))
\times (\times_{Y \in \V - \{X\}} \R(Y)) \rightarrow \R(X)$.

The structural equations define what happens in the presence of external
interventions. 
Setting the value of some variable $X$ to $x$ in a causal
model $M = (\Scal,\cF)$ results in a new causal model, denoted
$M_{X\gets x}$, which is identical to $M$, except that the
equation for $X$ in $\cF$ is replaced by $X = x$.

\emph{Probabilistic causal models}
are pairs $(M,\Pr)$, 
where $M$ is a causal model and $\Pr$ is a
probability on the contexts. A causal model  $M$ is \emph{recursive} (or \emph{acyclic})
if its causal graph is acyclic.
If $M$ is an acyclic  causal model,
then given a \emph{context}, that is, a setting $\vec{u}$ for the
exogenous variables in $\U$, the values of all the other variables are
determined.
In this paper we restrict to recursive models.

We call a pair $(M,\vec{u})$ consisting of a causal model $M$ and a
context $\vec{u}$ a \emph{(causal) setting}.
A causal formula $\psi$ is true or false in a setting.
We write $(M,\vec{u}) \sat \psi$  if
the causal formula $\psi$ is true in
the setting $(M,\vec{u})$.
The $\sat$ relation is defined inductively.
$(M,\vec{u}) \sat X = x$ if
the variable $X$ has value $x$
in the unique solution
to the equations in
$M$ in context $\vec{u}$.
Finally, 
$(M,\vec{u}) \sat [\vec{Y} \gets \vec{y}]\varphi$ if 
$(M_{\vec{Y} = \vec{y}},\vec{u}) \sat \varphi$,
where $M_{\vec{Y}\gets \vec{y}}$ is the causal model that is identical
to $M$, except that the 
variables in $\vec{Y}$ are set to $Y = y$
for each $Y \in \vec{Y}$ and its corresponding 
value $y \in \vec{y}$.

A standard use of causal models is to define \emph{actual causation}: that is, 
what it means for some particular event that occurred to cause 
another particular event. 
There have been a number of definitions of actual causation given
for acyclic models
(e.g., \cite{beckers21c,GW07,Hall07,HP01b,Hal19,hitchcock:99,Hitchcock07,Weslake11,Woodward03}).
In this paper, we focus on what has become known as the \emph{modified} 
Halpern-Pearl definition and some related definitions introduced
by Halpern ~\citeyear{Hal19}. 
We briefly review the relevant definitions below.
The events that can be causes are arbitrary conjunctions of primitive
events (formulas of the form $X=x$); 
the events that can be caused are primitive events, denoting the output of the model.  

\dfn\label{def:AC}[Actual cause]
$\vec{X} = \vec{x}$ is 
an \emph{actual cause} of $\varphi$ in $(M,\vec{u})$ if the
following three conditions hold: 
\begin{description}
\item[{\rm AC1.}]\label{ac1} $(M,\vec{u}) \models (\vec{X} = \vec{x})$ and $(M,\vec{u}) \models \varphi$. 
\item[{\rm AC2.}] There is a
  a setting $\vec{x}'$ of the variables in $\vec{X}$, a 
(possibly empty)  set $\vec{W}$ of variables in $\V - \vec{X}'$,
and a setting $\vec{w}$ of the variables in $\vec{W}$
such that $(M,\vec{u}) \models \vec{W} = \vec{w}$ and
$(M,\vec{u}) \models [\vec{X} \gets \vec{x}', \vec{W} \gets
    \vec{w}]\neg{\varphi}$, and moreover
\item[{\rm AC3.}] \label{ac3}\index{AC3}  
  $\vec{X}$ is minimal; there is no strict subset $\vec{X}'$ of
  $\vec{X}$ such that $\vec{X}' = \vec{x}''$ can replace $\vec{X} =
  \vec{x}'$ in 
  AC2, where $\vec{x}''$ is the restriction of
$\vec{x}'$ to the variables in $\vec{X}'$.
\end{description}
\edfn
In the special case that $\vec{W} = \emptyset$, we get the 
but-for definition.

The notion of explanation, taken from~\cite{Hal19}, is relative to a set of contexts.
\dfn\label{def:EX}[Explanation]
$\vec{X} = \vec{x}$ is 
an \emph{explanation} of $\varphi$ relative to a set $\K$ of contexts 
in a causal model $M$ if the following conditions hold:  
\begin{description}
\item[{\rm EX1a.}]  
If $\vec{u} \in \K$ and $(M,\vec{u}) \models (\vec{X} = \vec{x})
  \wedge \varphi$, then there exists a conjunct $X=x$ of $\vec{X} =
  \vec{x}$ and a (possibly empty) conjunction $\vec{Y} = \vec{y}$ such
  that $X=x \wedge \vec{Y} = \vec{y}$ is an actual cause of $\varphi$
  in $(M,\vec{u})$. 
\item[{\rm EX1b.}] $(M,\vec{u}') \models [\vec{X} = \vec{x}]\varphi$  for all
  contexts $\vec{
    u}' \in \K$. 
\item[{\rm EX2.}] $\vec{X}$ is minimal; there is no
  strict subset $\vec{X}'$ of $\vec{X}$ such that $\vec{X}' =
  \vec{x}'$ satisfies EX1,  
where $\vec{x}'$ is the restriction of $\vec{x}$ to the variables in $\vec{X}'$. (This is SC4).
\item[{\rm EX3.}] \label{ex3} $(M,u) \sat \vec{X} = \vec{x} \wedge
  \varphi$ for some $u \in \K$.
\end{description}
\edfn

\section{Theoretical Foundations of \multideepcover}\label{sec:theory}

We view an image classifier (\emph{e.g.} a neural network) as a
probabilistic causal model. 
Specifically, 
the endogenous variables are taken to be the set $\vec{V}$ of pixels
that the image classifier gets as input, together with an output
variable that we call $O$.  The variable $V_i \in \vec{V}$ describes
the color and 
intensity of pixel $i$; its value is determined by the exogenous variables.  
The equation for $O$ determines the output of the
model as a function of the pixel values.
Thus, the causal network has depth $2$, with the exogenous variables
determining the feature 
variables,
and the feature variables determining
the output variable.  
In this paper we also assume \emph{causal independence} between the feature
variables in $\vec{V}$, the set of pixels of the input image.

Pixel independence is a common assumption in explainability tools.
This is a non-trivial assumption, 
and it might seem far-fetched, especially if we consider images capturing
real objects: if a group of pixels captures, say, a cat's ear, then a
group of pixels near it should capture a cat's eye. However, we argue that
it is, in fact, accurate on images. Indeed, consider a 
partially obscured image, obtained by
overlaying random color patches over an Imagenet image. Such an image is perfectly valid, as obscuring a part of the input image either by introducing an artificial object 
or by positioning a real object in front of the primary subject of the classification 
does not lead to any change in unobscured pixels. For example, obscuring a cat's ear does not lead to
any change in an (unobscured) group of pixels, capturing a cat's eye. 
Thus, pixel independence holds on general images.

We refer the reader
to \cite{CH24} for a more in-depth discussion of causal independence between pixel values in
image classification.  
We note that the causal independence assumption is not true in
other types of inputs, such as tabular or spectral data. For those types of inputs, assuming independence
is clearly an approximation and might lead to inaccurate results.
In this paper, however, we focus
on images, where causal independence between pixels holds.

Moreover, as the causal network is of depth $2$, all parents
of the output variable $O$ are contained in $\vec{V}$.  
Given these assumptions, the probability on contexts directly
corresponds to the probability on seeing various images (which the model
presumably learns during training).

Given an input image $I$, the set of contexts $\K$ that we consider for an explanation
is the set $\K_I$ obtained by \emph{all partial occlusions of $I$}, where an occlusion sets
a part $\vec{Y}$ of the image to a predefined masking color $\vec{y}$. The probability
distribution over $\K_I$ is assumed to be uniform.

Under the assumptions above, the following definition is equivalent to
\Cref{def:EX}, as shown in~\cite{CH24}.
\begin{definition}[Explanation for image classification~\cite{CKS21}]\label{defn:simple-exp}
An explanation is a minimal subset of pixels of a
given input image that is sufficient for the model $\mathcal{N}$ to classify the image,
where ``sufficient'' is defined as containing only this subset of pixels
from the original image, with the other pixels set to the masking color.
\end{definition}

In the complexity discussion that follows, we discuss the complexity of decision problems matching the
function problems of computing explanations. If the decision problem is $A$-complete, for some complexity
class $A$, then the matching function problem is FP$^{A[\log{n}]}$-complete, assuming monotonicity of
the function problem (see \cite{CH04} for a more in-depth discussion of decision vs function problems in actual
causality). Formally, the decision problem of explanation is, given a
model, a context, an output $\varphi$, and a candidate explanation, to decide whether the candidate is indeed an explanation
for $\varphi$ in the given model and context.

\cite{CKS21} observe that the precise computation of an explanation in our
setting is intractable, as the problem is equivalent to an earlier
definition of explanations in binary causal models, which is
DP-complete~\cite{EL04}.\footnote{DP is the class of languages that are
an intersection of a language in NP and a language in co-NP and contains, in
particular, the languages of unique solutions to NP-complete
problems~\cite{Pap84}.}

The following lemma shows that computing a second (or any subsequent)
explanation is not easier than computing the first one.
\begin{lemma}\label{lemma:secondexp}
Given an input image and one explanation, the decision problem of a different explanation is DP-complete.
\end{lemma}
\begin{proof}
Membership in DP is straightforward and follows from the membership in DP of the problem of deciding an explanation: adding a constraint that 
the output should be different from a given explanation does not increase the complexity class of the decision problem.
For hardness in DP, we show a reduction from the decision problem of computing an explanation. Given an image $I$ classified by
a neural network (a black-box model) ${\cal N}$ as ${\cal N}(I)$, we define $\varphi$ as ``the output is an explanation of ${\cal N}(I)$
or it is exactly $I$''. The reduction is a tuple $\langle I, {\cal N}(I), I \rangle$, viewed as an input to the different explanation
problem, that is, $\langle$ input image, its label, a given explanation $\rangle$. 
The second $I$ renders $\varphi$ true. Then, an output that renders $\varphi$ true and is
different from $I$ is an explanation of ${\cal N}(I)$ for $I$, completing the reduction.
\end{proof}

\citet{chockler2024} use a greedy approach to constructing approximate explanations, based on scanning the ranked list of pixels $\mathit{pixel\_ranking}$ (\Cref{fig:algorithm}). 
The pixels are ranked in the order of their approximate \emph{degree of responsibility} for the classification,
where responsibility is a quantitative measure of causality and, roughly speaking, measures the amount of causal
influence on the classification. Formally, 
the \emph{degree of responsibility} of a variable $X=x$ for the value of $\varphi$ is $1/k$, where $k$ is the size of a smallest
set of variables $\vec{X}$ s.t. $X\in\vec{X}$ and has the value $x$ in $\vec{x}$, and $\vec{X}=\vec{x}$ is an actual cause of
$\varphi$ according to \Cref{def:AC}~\cite{CH04}. The degree of responsibility
is always between $0$ and $1$, with higher values indicating a stronger
causal influence. 

The precise computation of degrees of responsibility of pixels of $I$ is intractable; it's decision problem 
is NP-complete under our simplifying assumptions~\cite{CH04}. 
Hence, the ranking in~\cite{chockler2024} is based on the approximate degree of responsibility, which is computed by partitioning the set in iterations and computing the degrees of responsibility for each partition, while discarding low-responsibility elements (see \Cref{sec:algorithm} for details). The greedy explanation extraction
adds pixels from the sorted ranked list until the original classification is obtained.

However, reducing the complexity of computing one explanation does not reduce the complexity of computing many explanations, 
as the number of explanations for a given image can be very high: 
\begin{lemma}\label{lemma:number}
The number of explanations for an input image is bounded from above by $\binom{n}{\lfloor n/2 \rfloor}$, and this bound is tight.
\end{lemma}
\begin{proof}
Since an explanation of the classification of $x$ is a minimal subset of $x$ that is sufficient to result in the same classification, 
the number of explanations is characterised by \emph{Sperner's theorem}, which provides a bound for the number $S$ of largest possible families of finite sets,
none of which contain any other sets in the family~\cite{And87}.
By Sperner's theorem, $S \leq \binom{n}{\lfloor n/2 \rfloor}$, and the bound is reached when all subsets are of the size $\lfloor n/2 \rfloor$. 
The following example demonstrates an input on which this bound is reached.
Consider a binary classifier ${\cal N}$ that determines whether an input image of size $n$ has at least $\lfloor n/2 \rfloor$ 
green-colored pixels and an input image $I$ that is completely green. Then, each explanation is of size 
$\lfloor n/2 \rfloor$, and there are $\binom{n}{\lfloor n/2 \rfloor}$ explanations.
\end{proof}

Finally, we note that given a set of explanations (sets of pixels) and an overlap bound, deciding a subset 
of a given number of explanations in which elements overlap for no more than the bound is NP-hard even assuming that constructing and training a binary classifier is $O(1)$, by reduction from the independent set problem, which is known to be NP-complete.
Indeed, let $(G = \langle V,E \rangle, n)$ be an input to the independent set problem, deciding whether $G$ contains
an independent set of nodes of size $n$. Then, $G$ has an independent set of size $n$ if and only if there exist $n$ disjoint explanations
of $G$ having a connected component of size $1$ (note that finding a connected component of size $1$ is polynomial in the size of $G$).


\section{The \multideepcover Algorithm}
\label{sec:algorithm}

\begin{figure*}
\centering
\vspace{10mm}
\input{alg}
\caption{A schematic depiction of \multideepcover, returning a set of explanations $\mathcal{E}$ for a given input image. Its components: \ding{172} \emph{ranking}  generates a responsibility landscape of pixels; \ding{173} \emph{search} launches $x$ \emph{searchlight} searches over the landscape; \ding{174} \emph{drain} minimizes the explanations founds in \ding{173}; 
\ding{175} \emph{separate} produces a maximal subset $\mathcal{E}$ from the output of \ding{174}, with the given overlap bound.}
\vspace{3mm}
\label{fig:algorithm}
\end{figure*}

In this section we present our algorithm for computing multiple explanations of an image. As shown in \Cref{sec:theory}, the problem is intractable,
motivating the need for efficient and accurate approximation algorithms. Due to the lack of space, some details and algorithms have been moved to the supplementary material.

The concept of a \emph{superpixel} is used in a number of different explanation tools; \sag splits the image into a fixed grid of $7\times7$ superpixels.
Dividing the image in this way greatly reduces the computational cost of
searching for explanations. The rigidity of the grid, however, may lead to 
missing some explanations. \multideepcover does not need a rigid grid:
it starts with large, randomly selected superpixels which it iteratively refines. Furthermore, \multideepcover repeats this procedure with different starting superpixels to reduce the influence of particular random choices. Responsibility landscapes from individual iterations are combined to produce a detailed responsibility landscape~\Cref{subfig:peacock_rm}. As more iterations are added, the landscape becomes smoother. We separate explanations from this landscape using~\Cref{algo:searchlight}. It is the existence of this landscape which gives \multideepcover its edge: it provides a continuous search space over the entire image that is not confined to a discrete grid.

The high-level structure of the algorithm is presented in \Cref{fig:algorithm}, and the pseudo-code is in \Cref{algo:multiple_explanations}. We discuss each component in more detail below.

\begin{algorithm}
  \caption{$\mathit{\multideepcover}(I, \mathcal{N}, r, n, \delta, s, p, q)$}
  \label{algo:multiple_explanations}
  \begin{flushleft}
    \textbf{INPUT:}\,\, an image $I$, a model $\mathcal{N}$, a searchlight radius $r$, the maximum number of explanations $n$, maximum overlap between explanations $\delta$, number of searchlights $s$, searchlight expansions $p$, expansion coefficient $q$\\
    \textbf{OUTPUT:}\,\, a set of up to $n$ different explanations $\mathcal{E}$
  \end{flushleft}
  \begin{algorithmic}[1]
    \STATE $\mathcal{E} \leftarrow\emptyset$
    \STATE $l \leftarrow \mathcal{N}(x)$
    \STATE $\mathcal{S} \leftarrow \mathit{CAUSAL\_RANK}(I, \mathcal{N}, l)$\label{line:rank}
    \FOR{$i$ in $0\ldots s-1$} 
        \STATE $E_i \leftarrow \mathit{searchlight}(I,\mathcal{N}, l, \mathcal{S}, r, n, p, q)$\label{line:fl}
        \STATE $E_i \leftarrow \mathit{minimize}(I, E_i, \mathcal{N}, \mathcal{S})$\label{line:min}
        \STATE $\mathcal{E} \leftarrow \mathcal{E} \cup E_i$
     \ENDFOR 
     \STATE $\mathcal{E} \leftarrow \mathit{separate}(\mathcal{E}, \delta)$
     \STATE \textbf{return} $\mathcal{E}$
  \end{algorithmic}
\end{algorithm}

The $\mathit{CAUSAL\_RANK}$ procedure in Line 3 of~\Cref{algo:multiple_explanations} constructs a $\mathit{pixel\_ranking}$, which is a ranking of the pixels of the input image~$I$ by their causal responsibility. We use the algorithm in~\cite{chockler2024} as the basis for producing this landscape.
The number of required explanations is given as an input parameter to the procedure, as the total number of explanations can be exponential (see \Cref{lemma:number}).

The $\mathit{Search}$ procedure called in Line~\ref{line:fl} is described in~\Cref{algo:searchlight}. 
It replaces the greedy explanation generation in~\rex with a modified stochastic hill climb. In contrast to most hill-climb-based algorithms that look for the global maximum, we specifically search for \emph{local maxima}, on the assumption that these are likely to correspond to explanations. The function $\mathit{initialize}$ in Line~\ref{line:init} creates a `searchlight', $\mathcal{F}$, of radius $r$ at a random position over the image $I$. The intuition is that only the pixels under $\mathcal{F}$ are exposed to the model (all other pixels being set to a baseline value). 

If $\mathcal{F}$ contains an explanation (\ie the exposed pixels already have the required classification), we invoke the \emph{minimize} procedure to remove redundant pixels (Line~6). The reason for potential redundancy is that $\mathcal{F}$ may have included too many pixels by virtue of its circular shape. 
The \emph{minimize} procedure consists of obtaining the responsibility ranking of all pixels inside $\mathcal{F}$, setting all those pixels outside $\mathcal{F}$ to responsibility $0$, and using the greedy algorithm from~\cite{chockler2024} to add in pixels based on their responsibility. The stopping condition is the desired classification. This process is guaranteed to succeed because $\mathcal{F}$ already contains \emph{at least} the sufficient pixels. Pixels with $0$ responsibility are never added to an explanation.

A randomly placed searchlight might not contain an explanation, either due to its size or to its location. 
Rather than changing the location of $\mathcal{F}$ immediately, we first increase its size. 
If we start too small, we might miss all explanations (that is, $\mathcal{F}$ would never be large enough to capture all the necessary pixels). 
The number of expansions and size of expansion are controlled by hyperparameters.

If the increase in the size of $\mathcal{F}$ still does not result in an explanation, the algorithm changes its location and resets
to the original size. This step is guided by an objective function; by default, the objective function is the mean of the responsibility of the pixels  contained in $\mathcal{F}$. Thus, $\mathcal{F}$ moves towards the areas with a higher average responsibility; 
these areas are more likely to contain an explanation.

\begin{algorithm}[t]
    \caption{$\mathit{searchlight}\,(\mathit{I}, \mathcal{N}, l, \mathcal{S}, \mathit{r}, n, p, q)$}
    \label{algo:searchlight}
    \begin{flushleft}
        \textbf{INPUT:}\,\, an image $I$, a model $\mathcal{N}$, a label $l$, a responsibility landscape $\mathcal{S}$, a searchlight radius $r$, number of steps $n$, number of expansions $p$, radius increase $q$\\
        \textbf{OUTPUT:}\,\, an explanation $E$
    \end{flushleft}     
        \begin{algorithmic}[1]
            \STATE $\mathcal{F} \leftarrow \mathit{initialize(r)}$\label{line:init}
            \STATE $\mathcal{E} \leftarrow \emptyset$
            \FOR{$i$ in $0\ldots n-1$}
                \FOR{$j$ in $0\ldots p-1$}
                    \STATE $l' \leftarrow \mathcal{N}(\mathcal{F}(I))$
                    \IF{$l = l'$}
                        \STATE $E \leftarrow \mathcal{F}(I)$
                        \STATE \textbf{return} $E$
                    \ELSE
                        \STATE $\mathcal{F} \leftarrow \mathit{expand\_radius(r * q)}$
                    \ENDIF    
                \ENDFOR 
                \STATE $\mathcal{F} \leftarrow \mathit{neighbor}$
            \ENDFOR
            \STATE \textbf{return} $\mathcal{E}$
        \end{algorithmic}
\end{algorithm}

\begin{algorithm}[t]
     \caption{$\mathit{separate}(\mathcal{E}, \delta)$}
      \label{algo:cleanup}
      \begin{flushleft}
         \textbf{INPUT:}\,\, a set of explanations $\mathcal{E}$, a permitted degree of overlap $\delta$ \\
         \textbf{OUTPUT:}\,\, a subset of explanations $\mathcal{E'} \subseteq \mathcal{E}$ with overlap at most $\delta$
         \begin{algorithmic}[1]
             \STATE $\mathit{all\_pairs} \leftarrow \mathcal{E} \times \mathcal{E}$\label{line:allpairs}
             \STATE $\mathit{bad\_pairs} \leftarrow \emptyset$ 
             \FOR{$\mathit{(p_i, p_j)}$ in $\mathit{all\_pairs}$}
                 \STATE $\mathit{SDC} \leftarrow \mathit{dice\_coefficient(p_i, p_j)}$
                 \IF {$SDC > \delta$}
                    \STATE $\mathit{bad\_pairs} \leftarrow \mathit{bad\_pairs} \cup (p_i, p_j)$
                  \ENDIF
             \ENDFOR

             \FOR{$\mathit{e} \in \mathcal{E}$}
                \IF {$\mathit{e}$ does not contain any $\mathit{bad\_pairs}$}
                    \STATE \textbf{return} $\mathit{e}$
                \ENDIF
             \ENDFOR
         \STATE \textbf{return} $\emptyset$
         \end{algorithmic}
     \end{flushleft}
\end{algorithm}

Finally, the $\mathit{separate}$ procedure (\Cref{algo:cleanup}) separates a subset of at most $n$
explanations that overlap on pixels up to the bound $\delta$. $\delta$ is a value between $0$ and $1$, 
where $0$ stands for ``no permitted overlap'', and $1$ means ``no overlap restrictions''.
As discussed in \Cref{sec:theory}, the exact solution is NP-hard. The \emph{separate} procedure uses a greedy heuristic based on 
the S\o{}rensen–Dice coefficient (SDC)~\cite{Dic45,Sor48},  typically used as a measure of similarity between samples. 
First, we create a list of all pairs in $\mathcal{E}$ which overlap by more that $\delta$. We then iterate backwards through the powerset of $\mathcal{E}, 2^{\mathcal{E}}$ (\ie starting from the complete $\mathcal{E}$ and not $\emptyset$) and stop at the first set $\mathit{e} \in 2^{\mathcal{E}}$ which does not contain one of the previously discovered `bad' pairs. As an added optimization, when iterating through $2^{\mathcal{E}}$, we order all subsets $\mathit{e}$ with the same cardinality by the total area of the contained explanations. Thus, the algorithm stops at the largest number of explanations with the smallest overall area and with overlap less that $\delta$.

\section{Experimental Results}\label{sec:experiments}

\begin{table*}[t]
    \begin{subtable}{.45\textwidth}\centering
         \begin{tabular}{r|r|r||r|r||r|r}
            \toprule\toprule
            \multirow{2}{2em}{No. Exp} & \multicolumn{6}{c}{Datasets} \\
            \cmidrule(lr){2-7} 
            \multirow{2}{2em}{} & \multicolumn{2}{c}{ImageNet1k} & \multicolumn{2}{c}{Voc} & \multicolumn{2}{c}{ECSSD} \\
            \cmidrule(lr){2-7} 
            & \small{\textsc{Mult}} & \small{\sag} & \small{\textsc{Mult}} & \small{\sag} & \small{\textsc{Mult}}  & \small{\sag} \\
            \midrule\midrule 
            1 & 2052  & 2757 & 931 & 1138 & 548 & 732 \\
            \midrule
            2 & 1216 & 733 & 376 & 223 & 315 & 185  \\
            3 & 489  & 246 & 125 & 64 & 106 & 53 \\
            4 & 135  & 103 & 15 & 13 & 25 & 15  \\
            5 & 30   & 42 & 2 & 5 & 3 & 9 \\
            6 & 1    & 18 & -- & 2 & 3 & 3   \\
            7 & --    & 10 & -- & 2 & -- & 1    \\
            8+ & --    & 14 & -- & 2 & -- & 2   \\
            \bottomrule\bottomrule
        \end{tabular}
        \caption{Default \sag and \multideepcover with at least $90\%$ of original confidence and for the top classification. More than $1$ explanation is good.}
        \label{tab:resnet50_filter}
        \end{subtable}%
        \hfill
    \begin{subtable}{.45\textwidth}\centering
         \begin{tabular}{r|r|r||r|r||r|r}
    \toprule\toprule
    \multirow{2}{2em}{No. Exp} & \multicolumn{6}{c}{Datasets} \\ 
    \cmidrule(lr){2-7} 
    \multirow{2}{2em}{} & \multicolumn{2}{c}{ImageNet1k} & \multicolumn{2}{c}{Voc} & \multicolumn{2}{c}{ECSSD} \\
    \cmidrule(lr){2-7} 
            & \small{\textsc{Mult}} & \small{\sag} & \small{\textsc{Mult}} & \small{\sag} & \small{\textsc{Mult}}  & \small{\sag} \\
     \midrule\midrule 
        1 & 333  & 517 & 291 & 263 & 158 & 112 \\
        \midrule
        2 & 688 & 602 & 423 & 257 & 248 & 150  \\
        3 & 1048  & 651 & 411 & 251 & 295 & 147 \\
        4 & 974  & 540 & 222 & 180 & 171 & 173  \\
        5 & 571   & 430 & 77 & 142 & 90 & 121 \\
        6 & 246    & 276 & 22 & 79 & 32 & 66   \\
        7 & 52    & 223 & 3 & 63 & 6 & 57    \\
        8+ & 11    & 684 & -- & 214 & -- & 174   \\
    \bottomrule\bottomrule
    \end{tabular}
    \caption{\multideepcover with default parameters; \sag with a probability threshold of $0.01$. 
    \sag produces more explanations, but they are not for the top classifications (see \Cref{tab:resnet50_no_filter_pos}),
    rendering this comparison non-informative.}
    \label{tab:resnet50_no_filter}
    \end{subtable}%
    \vspace{4mm}
    \caption{\sag and \multideepcover on ResNet50. \Cref{tab:resnet50_filter} shows \sag with default probability threshold of $0.9$. We force \multideepcover to use the same confidence threshold, while still producing the top classification. \multideepcover finds more explanations per image in general than \sag. \Cref{tab:resnet50_no_filter} shows the effect of removing \sag{}'s probability threshold, to bring it closer to \multideepcover{}'s behavior. \multideepcover still performs well, but \sag{}'s output is close to noise.}
    \vspace{5mm}
\end{table*}

\begin{table*}[ht]
    \begin{subtable}{.45\linewidth}\centering
        \begin{tabular}{r|r|r|r}
    \toprule\toprule
    \multirow{2}{4.8em}{Position} & \multicolumn{3}{c}{Number of \sag Explanations} \\
    \cmidrule(lr){2-4} 
    & ImageNet & VOC & ECSSD \\ 
    \midrule\midrule 
        0 & 3698  & 1221 & 860 \\
        1 & 181 & 166 & 111 \\
        2 & 28  & 150 & 22 \\
        3 & 13  & 7 & 5  \\
        4 & 2   & 4  & 1  \\
        5+ & 1   & 1  & 1  \\
    \midrule\midrule
     Total in position $0$ (\%) & $\mathbf{94\%}$ & $\mathbf{84\%}$ & $\mathbf{86\%}$ \\
    \bottomrule\bottomrule
    \end{tabular}
    \caption{Position in the output tensor for ResNet50 explanations as produced by \sag at $0.9$ probability threshold.}
    \label{tab:resnet50_filter_pos}
    \end{subtable}%
    \hfill
    \begin{subtable}{.45\linewidth}\centering
       \begin{tabular}{r|r|r|r}
    \toprule\toprule
    \multirow{2}{4.8em}{Position} & \multicolumn{3}{c}{Number of \sag Explanations} \\
    \cmidrule(lr){2-4} 
    & ImageNet & VOC & ECSSD \\ 
    \midrule\midrule 
        0 & 94  & 20 & 13 \\
        1 & 69 & 33  & 17 \\
        2 & 51  & 20 & 11 \\
        3 & 63  & 17 & 8  \\
        4 & 82   & 28  & 11  \\
        5+ & 3565   & 1331 & 940 \\
    \midrule\midrule
    Total in position $0$ (\%) & $\mathbf{2\%}$ & $\mathbf{1\%}$ & $\mathbf{1\%}$ \\
    \bottomrule\bottomrule
    \end{tabular}
    \caption{Position in the output tensor for ResNet50 explanations as produced by \sag at $0$ probability threshold.}
    \label{tab:resnet50_no_filter_pos}
    \end{subtable}
    \vspace{5mm}
    \caption{Positions of \sag{}'s explanations in the output of a ResNet50 model. The tables illustrate \sag's strong
    dependence on the probability threshold parameter: when it is $0.9$, the majority of \sag's explanations (though not all)
    are for the top classification, as shown in \Cref{tab:resnet50_filter_pos}; when the probability threshold is close to $0$, 
    \sag's output is for classifications further down on the list, resulting in nonsensical explanations. For values between
    $0.9$ and $0$ \sag's output is between these two extremes.}\label{tab-sag-param}
    \vspace{5mm}
\end{table*}

\paragraph*{Implementation}

We implemented~\Cref{algo:multiple_explanations}
in the tool \multideepcover for generating multiple causal explanations. 
Given a responsibility landscape, by default, \multideepcover attempts to find $10$ explanations (a parameter).
In practice, it is computationally inexpensive to find more explanations, but our experimental results demonstrate
that images with more than $6$ explanations are rare ($\approx1\%$ of images).

\paragraph*{Comparison to \sag}
As discussed earlier, \sag{}'s definition of explanation is sufficiently different as to make exact comparison difficult.
\sag uses a beam search over a $7\times 7$ grid to discover multiple explanations. 
This strategy results in the search space of size $2^{49}$
(the number of subsets of regions of the \sag grid), and \sag attempts to solve the
exponential explosion problem by reducing the search space \emph{at random}. 

Furthermore, \sag accepts a region as an explanation if the confidence of the original label, $l$, on this region is greater than a confidence bound, regardless of its position in the model's output tensor. That is, \sag might output an explanation to an entirely different label than
the top classification of the input.
The confidence bound of \sag is based on a hyperparameter `probability threshold' and confidence on the original image.
This is in contrast with \multideepcover's much stricter definition of explanation, where a subset of pixels constitute an explanation is they alone are classified $l$ by the model, where $l$ is the top classification, regardless of the model's confidence. 
In order to perform a fair comparison, we evaluate both tools twice, each time with the default settings of one of them. 
Namely, we run \sag with default settings (specifically, with the probability threshold $0.9$) 
and compare it against \multideepcover where we set a minimum confidence threshold for an explanation at $0.9$ as well. 
Note that is still not an entirely fair comparison, as \multideepcover always explains the top classification, contrary to \sag. 
We also perform the comparative evaluation with setting \sag{}'s probability threshold to $0$\footnote{In reality, we found setting it to $0.0$ precisely resulted in no explanations, so for our experiment we used the value $0.01$} 
and leaving \multideepcover{}'s minimum confidence threshold at its default value of $0$. 

We set the maximum number of explanations in both tools to $10$. \sag requires the user to set the maximal allowed overlap in grid squares, with suggested values of $0$, $1$, or $2$.
\multideepcover does not measure overlap in blocks, as its responsibility landscape is continuous (or, rather, discrete at the level of a single
pixel). For a fair comparison, we limit both tools to non-overlapping explanations.
Apart from these changes, the experiments are performed on \multideepcover and \sag with default settings. 

\paragraph*{Datasets and Models}
For our experiments, we use $3$ different models from \textsc{torchvision} with default weights. The models are `resnet50', `convnext\_large' and `vit\_b\_32'. We use $3$ different publicly available datasets: ImageNet-mini validation\footnote{https://www.kaggle.com/datasets/ifigotin/imagenetmini-1000}, Pascal VOC2012~\cite{pascal-voc-2012} and ECSSD~\cite{ecssd}.

\paragraph*{Experimental Results}
A natural and quantifiable performance measure for multiple explanations is the number of significantly different explanations produced for each image. We also consider, for \sag, the position of the explanation in the model output (for \multideepcover, this position is always $0$,
as \multideepcover always explains the actual (top) classification). We also consider the size of explanations: in general, a good explanation should be close to minimal, \ie having as few extraneous pixels as possible. This has an important bearing on multiple explanations if we do not have overlap: one cannot fit many large explanations in a standard-size image.

The experiment was performed on a 64-core machine running Ubuntu 20.04.6 with a 48GB RAM and several Nvidia A40 GPUs. For the sake of space, we present complete results for ResNet50 in the paper and other results in the supplementary material. The results for the other models follow the same pattern as for ResNet50. The `vit\_b\_32' model accepts a lower number of multiple explanations for both tools (see supplementary material). 

\Cref{tab:resnet50_filter} shows the comparison of \sag with default parameters against \multideepcover with a confidence threshold of $0.9$. \Cref{tab:resnet50_filter_pos} shows the positions of \sag{}'s explanations in the model output. It seems unreasonable to accept a set of pixels explaining a classification at the $6^{th}$ position as an explanation for the top classification.

At threshold $0.9$, \sag{}'s output mostly (though not always) refers to the top classification, 
and \multideepcover finds more multiple explanations in general. \sag's dependence on the probability threshold parameter 
is illustrated in
\Cref{tab-sag-param}. With the threshold close to $0$, one of the \sag's
explanations was in position $831$ out of the possible thousand positions for an input in ImageNet. 


The average size of an explanation for \multideepcover is $\approx 8\%$ of the input image across the $3$ datasets. In contrast, the average size of \sag's explanations when run with $0.9$ threshold 
(hence producing explanations mostly for the top classification) is $\approx 14\%$ of the input image, showing that
\multideepcover produces much tighter explanations.


\multideepcover, unlike \sag, is consistent in its explanations: small and the top classification, regardless of probability threshold. If we increase the probability threshold for \multideepcover, the explanations become a little larger. We do not have the problem seen by~\cite{Jiang_2024_CVPR}: a plethora of `explanations' with little or no explanatory power. It becomes much clearer that the `compositionality' of the different models is significantly different. In particular, the decision process of
ResNet50 is similar to the one of ConvNext, and both are significantly different from that of the ViT model.

\paragraph*{Timings} \multideepcover takes an average of $15$ seconds per image on our machine, for $20$ iterations of the main algorithm. \sag, with the default probability threshold $0.9$, takes on average $10$ seconds; 
with a threshold of $0.01$, \sag is faster, at $2$ to $3$ seconds, but produces uninformative results.
\section{Conclusions}
\label{sec:conclusions}

We have introduced \multideepcover, a novel explanation-discovery algorithm based on the responsibility landscape constructed in the ranking procedure
and a ``spotlight'' search, ensuring different spatial locations for explanations. 

\newpage
\bibliography{all}

\newpage








\appendix

\section{Code and Data}
All data and experiments can be found at the following link: \url{https://figshare.com/s/37e6736e52fb0f92adf3}.

\multideepcover has been incorporated into \rex on GitHub at \url{https://github.com/ReX-XAI/} and is on PyPI as \url{https://pypi.org/project/rex-xai/}. 

The tool \sag can be found at \url{https://github.com/viv92/structured-attention-graphs}.

\section{Other models}

\begin{table}[]
    \centering
    \begin{tabular}{r|r|r|r|r}
    \toprule\toprule
    \multirow{2}{3em}{No.\\ Exp} & \multicolumn{3}{c}{Tools} \\
    \cmidrule(lr){2-5} 
    & \multideepcover & \sag 0.9 & \sag 0.4 & \sag 0.1\\ 
    \midrule\midrule 
        1 & 151  & 3243 & 2476 & 1532 \\
        2 & 2707 & 489 & 816  & 1012 \\
        3 & 826  & 123 & 345 & 566\\
        4 & 203  & 48  & 155 & 354 \\
        5 & 24   & 12  & 69  & 193  \\
        6 & 7    & 4  & 34  & 120 \\
        7 & 2    & 0   & 19 & 58   \\
        8 & 2    & 1   & 3  & 31\\
        9 & 1    & 0   & 1 & 32  \\      
        10 & 0   & 3   & 0  & 25 \\
    \midrule\midrule
     \% & 100\% & 100\% & 100\% & 100\%\\
    \bottomrule\bottomrule
    \end{tabular}
    \vspace*{2mm}
    \caption{Number of disjoint explanations produced by \multideepcover against \sag on VGG19.}
    \label{tab:results_baseline}
\end{table}
\Cref{tab:results_baseline} shows results on VGG19 on ImageNet-mini validation. The effect of \sag{}'s probability threshold on the number of results found is very clear. When the threshold is taken down to $0.1$, the number of (multiple) explanations found is, unsurprisingly, much higher. However, if we limit \sag to considering explanations which are only in top position, we find that only 20\% of \sag $0.1$ explanations are valid.

\begin{table*}[t]
    \begin{subtable}{.5\textwidth}\centering
         \begin{tabular}{r|r|r||r|r||r|r}
            \toprule\toprule
            \multirow{2}{2em}{No. Exp} & \multicolumn{6}{c}{Datasets} \\
            \cmidrule(lr){2-7} 
            \multirow{2}{2em}{} & \multicolumn{2}{c}{ImageNet1k} & \multicolumn{2}{c}{Voc} & \multicolumn{2}{c}{ECSSD} \\
            \cmidrule(lr){2-7} 
            & \small{\textsc{Mult}} & \small{\sag} & \small{\textsc{Mult}} & \small{\sag} & \small{\textsc{Mult}}  & \small{\sag} \\
            \midrule\midrule 
            1 & 3699  & 3516 & 1401 & 1138 & 950 & 918 \\
            2 & 181 & 328 & 34 & 223 & 35 & 67  \\
            3 & 27  & 55 & 5 & 64 & 7 & 9 \\
            4 & 3  & 11 & 1 & 13 & 2 & 4  \\
            5 & 2   & 6 & -- & 5 & 2 & -- \\
            6 & --    & 3& -- & 2 & -- & 1   \\
            7 & --    & 2 & -- & 2 & -- & --    \\
            8+ & --    & 1 & -- & 2 & -- & 1   \\
            \bottomrule\bottomrule
        \end{tabular}
        \caption{\sag with default parameters, \multideepcover with at least $90\%$ \\of the original confidence and top classification.}
        \label{tab:vit_filter}
        \end{subtable}%
        \hfill\hfill
    \begin{subtable}{.5\textwidth}\centering
         \begin{tabular}{r|r|r||r|r||r|r}
    \toprule\toprule
    \multirow{2}{2em}{No. Exp} & \multicolumn{6}{c}{Datasets} \\ 
    \cmidrule(lr){2-7} 
    \multirow{2}{2em}{} & \multicolumn{2}{c}{ImageNet1k} & \multicolumn{2}{c}{Voc} & \multicolumn{2}{c}{ECSSD} \\
    \cmidrule(lr){2-7} 
            & \small{\textsc{Mult}} & \small{\sag} & \small{\textsc{Mult}} & \small{\sag} & \small{\textsc{Mult}}  & \small{\sag} \\
     \midrule\midrule 
        1 & 2689  & 1432 & 1224 & 585 & 744 & 342 \\
        2 & 782 & 971 & 157 & 400 & 171 & 289  \\
        3 & 305  & 587 & 42 & 193 & 57 & 147 \\
        4 & 97  & 354 & 7 & 106 & 17 & 88  \\
        5 & 33   & 194 & 5 & 58 & 7 & 49 \\
        6 & 13   & 112 & 1 & 32 & 3 & 23   \\
        7 & 3    & 66 & -- & 18 & 1 & 13    \\
        8+ & --    & 207 & -- & 57 & -- & 49   \\
    \bottomrule\bottomrule
    \end{tabular}
    \caption{\multideepcover with default parameters, while \sag has a probability threshold of $0.01$.}
    \label{tab:vit_no_filter}
    \end{subtable}%
    \vspace{5mm}%
    \caption{Positions of \sag{}'s explanations in the output of a `vit\_b\_32' model. When the probability threshold is $0.9$, \sag performs relatively well (\Cref{tab:vit_filter_pos}). When the probability threshold is removed for \multideepcover comparison, the output position of the explained classification slides further down on the list.}
    \vspace{5mm}
\end{table*}

\Cref{tab:vit_no_filter} shows the complete results on `vit\_b\_32', while \Cref{tab:vit_no_filter_pos} shows the 
positions of \sag explanations with $0$ probability threshold. \Cref{tab:vit_filter} 
shows the complete results on `vit\_b\_32', while \Cref{tab:vit_filter_pos} shows the positions of \sag explanations with $0.9$ probability threshold. Neither \multideepcover nor \sag finds as many explanations as with ResNet50. If we exclude \sag{}'s sub-top classifications, the performance between \sag and \multideepcover at $0.9$ is very similar. The performance at $0.0$ is comparable to ResNet50: \multideepcover still finds top classifications, \sag is close to uniformly distributed. Indeed, one set of pixels was in position $315$ in the output, but still counted as an explanation.

\begin{table*}[t]
    \begin{subtable}{.45\linewidth}\centering
        \begin{tabular}{r|r|r|r}
    \toprule\toprule
    \multirow{2}{4.8em}{Position} & \multicolumn{3}{c}{No. \sag Explanations} \\
    \cmidrule(lr){2-4} 
    & ImageNet & VOC & ECSSD \\ 
    \midrule\midrule 
        0 & 3801  & 1334 & 918 \\
        1 & 102 & 93 & 61 \\
        2 & 18  & 15 & 18 \\
        3 & 2  & 6 & 2  \\
        4 & --   & 1  & 1  \\
        5+ & --   & --  & --  \\
    \midrule\midrule
     Total in $0$ (\%) & 97\% & 92\% & 0.92\% \\
    \bottomrule\bottomrule
    \end{tabular}
    \caption{Position in the output tensor for ViT explanations as produced by \sag at $0.9$ probability threshold.}
    \label{tab:vit_filter_pos}
    \end{subtable}%
    \hfill
    \hfill
    \begin{subtable}{.45\linewidth}\centering
       \begin{tabular}{r|r|r|r}
    \toprule\toprule
    \multirow{2}{4.8em}{Position} & \multicolumn{3}{c}{No. \sag Explanations} \\
    \cmidrule(lr){2-4} 
    & ImageNet & VOC & ECSSD \\ 
    \midrule\midrule 
        0 & 107  & 30 & 21 \\
        1 & 135 & 31  & 33 \\
        2 & 140  & 30 & 32 \\
        3 & 153  & 32 & 26  \\
        4 & 186   & 46  & 28  \\
        5+ & 3202   & 1280 & 860 \\
    \midrule\midrule0
    Total in $0$ (\%) & 2\% & 2\% & 2\% \\
    \bottomrule\bottomrule
    \end{tabular}
    \caption{Position in the output tensor for ViT explanations as produced by \sag at $0$ probability threshold.}
    \label{tab:vit_no_filter_pos}
    \end{subtable}
    \vspace{5mm}
    \caption{Positions of \sag{}'s explanations in the output of a ViT model. When the probability threshold is $0.9$, \sag performs relatively well (\Cref{tab:vit_filter_pos}). When the probability threshold is removed for \multideepcover comparison, the output positions suffer greatly.}
    \vspace{5mm}
\end{table*}

\Cref{tab:conv_no_filter} shows the complete results on `vit\_b\_32', while \Cref{tab:conv_no_filter_pos} shows the positions of \sag explanations with $0$ probability threshold.

\begin{table*}[t]
    \begin{subtable}{.5\textwidth}\centering
         \begin{tabular}{r|r|r||r|r||r|r}
            \toprule\toprule
            \multirow{2}{2em}{No. Exp} & \multicolumn{6}{c}{Datasets} \\
            \cmidrule(lr){2-7} 
            \multirow{2}{2em}{} & \multicolumn{2}{c}{ImageNet1k} & \multicolumn{2}{c}{Voc} & \multicolumn{2}{c}{ECSSD} \\
            \cmidrule(lr){2-7} 
            & \small{\textsc{Mult}} & \small{\sag} & \small{\textsc{Mult}} & \small{\sag} & \small{\textsc{Mult}}  & \small{\sag} \\
            \midrule\midrule 
            1 & 3423  & 2829 & 1266 & 1154 & 823 & 708 \\
            2 & 408   & 683  & 138  & 202 & 143 & 181  \\
            3 & 68    & 243  & 39   & 59 & 24 & 80 \\
            4 & 21    & 129  & 6    & 19 & 6 & 21  \\
            5 & 1     & 22   & -- & 11 & 3 & 5 \\
            6 & --    & 11   & -- & 3 & 1 & 4   \\
            7 & --    & 3   & -- & 1 & -- & 1    \\
            8+ & --    & 3  & -- & -- & -- & --   \\
            \bottomrule\bottomrule
        \end{tabular}
        \caption{\sag with default parameters, \multideepcover with at least $90\%$ \\of the original confidence and top classification.}
        \label{tab:conv_filter}
        \end{subtable}%
        \hfill
    \begin{subtable}{.5\textwidth}\centering
         \begin{tabular}{r|r|r||r|r||r|r}
    \toprule\toprule
    \multirow{2}{2em}{No. Exp} & \multicolumn{6}{c}{Datasets} \\ 
    \cmidrule(lr){2-7} 
    \multirow{2}{2em}{} & \multicolumn{2}{c}{ImageNet1k} & \multicolumn{2}{c}{Voc} & \multicolumn{2}{c}{ECSSD} \\
    \cmidrule(lr){2-7} 
            & \small{\textsc{Mult}} & \small{\sag} & \small{\textsc{Mult}} & \small{\sag} & \small{\textsc{Mult}}  & \small{\sag} \\
     \midrule\midrule 
        1 & 85  & 289 & 127 & 143 & 83 & 69 \\
        2 & 369 & 393 & 284 & 220 & 175 & 97  \\
        3 & 724  & 467 & 363 & 202 & 235 & 119 \\
        4 & 941  & 484 & 355 & 182 & 249 & 131  \\
        5 & 957   & 469 & 229 & 174 & 157 & 118 \\
        6 & 568    & 329 & 70 & 138 & 78 & 88   \\
        7 & 221   & 268 & 19 & 63 & 22 & 84    \\
        8+ & 58    & 1224 & 2 & 327 & 1 & 57   \\
    \bottomrule\bottomrule
    \end{tabular}
    \caption{\multideepcover with default parameters, while \sag has a probability threshold of $0.01$.}
    \label{tab:conv_no_filter}
    \end{subtable}%
    \vspace{5mm}%
   \caption{Positions of \sag{}'s explanations in the output of a `convnext\_large' model. When the probability threshold is $0.9$, \sag performs relatively well (\Cref{tab:conv_filter_pos}). When the probability threshold is removed for \multideepcover comparison, the output position of the explained classification slides further down on the list.}
    \vspace{5mm}
\end{table*}

\begin{table*}[t]
    \begin{subtable}{.45\linewidth}\centering
        \begin{tabular}{r|r|r|r}
    \toprule\toprule
    \multirow{2}{4.8em}{Position} & \multicolumn{3}{c}{No. \sag Explanations} \\
    \cmidrule(lr){2-4} 
    & ImageNet & VOC & ECSSD \\ 
    \midrule\midrule 
        0 & 3891  & 1390 & 955 \\
        1 & 29 & 56 & 39 \\
        2 & 2  & 3 & 5 \\
        3 & 1  & -- & 1  \\
        4 & --   & --  & --  \\
        5+ & --   & --  & --  \\
    \midrule\midrule
     Total in $0$ (\%) & 94\% & 84\% & 0.86\% \\
    \bottomrule\bottomrule
    \end{tabular}
    \caption{Position in the output tensor for `convnext\_large' explanations as produced by \sag at $0.9$ probability threshold.}
    \label{tab:conv_filter_pos}
    \end{subtable}%
    \hfill
    \begin{subtable}{.45\linewidth}\centering
       \begin{tabular}{r|r|r|r}
    \toprule\toprule
    \multirow{2}{4.8em}{Position} & \multicolumn{3}{c}{No. \sag Explanations} \\
    \cmidrule(lr){2-4} 
    & ImageNet & VOC & ECSSD \\ 
    \midrule\midrule 
        0 & 48  & 12 & 6 \\
        1 & 71 & 17  & 11 \\
        2 & 93  & 26 & 9 \\
        3 & 146  & 38 & 8  \\
        4 & 156   & 41  & 37  \\
        5+ & --   & -- & -- \\
    \midrule\midrule
    Total in $0$ (\%) & 1\% & $<1$\% & $<1$\% \\
    \bottomrule\bottomrule
    \end{tabular}
    \caption{Position in the output tensor for `convnext\_large' explanations as produced by \sag at $0$ probability threshold.}
    \label{tab:conv_no_filter_pos}
    \end{subtable}
    \vspace{5mm}
    \caption{Positions of \sag{}'s explanations in the output of a ConvNext model. When the probability threshold is $0.9$, \sag performs relatively well (\Cref{tab:conv_filter_pos}). When the probability threshold is removed for \multideepcover comparison, the output positions suffer greatly.}
    \vspace{5mm}
\end{table*}

The results on `convnext\_large' are revealing for compositionality. The number of explanations which are both the top classification and have at least $90\%$ of original confidence is quite low compared to no threshold. This indicates that, while the model does not require a large number of pixels to make a classification, it \emph{does} require a large number in order to have high confidence.

The results over all three models show that \multideepcover{}'s output is not contingent on any probability thresholding, unlike with \sag. While \multideepcover does not always produce as many explanations as \sag for all models and datasets, those explanations it finds are higher quality in that they are smaller, always the most likely classification and can have a tunable confidence threshold does not affect the fundamental requirements of causal explanation.

\section{Occluded Images}\label{appendix:occluded}

\begin{figure*}[t]
    \centering
    \begin{subfigure}{\textwidth}
        \centering
        \includegraphics[scale=0.4]{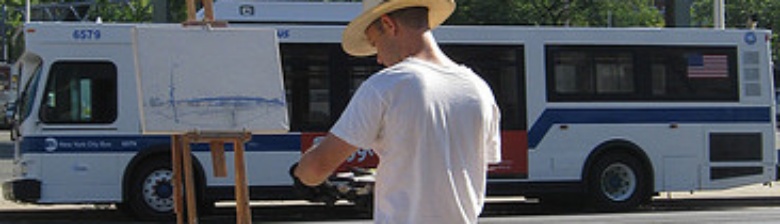}
        \caption{bus}
        \label{occ:bus}
        \vspace{1cm}
    \end{subfigure}
    \hfill
    \begin{subfigure}{0.2\textwidth}
        \centering
        \includegraphics[scale=0.3]{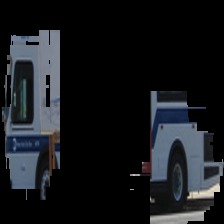}
        \caption{\multideepcover explanation}
        \label{occ:one}
    \end{subfigure}
    \hfill
     \begin{subfigure}{0.2\textwidth}
        \centering
        \includegraphics[scale=0.3]{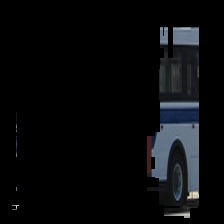}
        \caption{\multideepcover explanation}
        \label{occ:two}
    \end{subfigure}
    \hfill
      \begin{subfigure}{0.2\textwidth}
        \centering
        \includegraphics[scale=0.09]{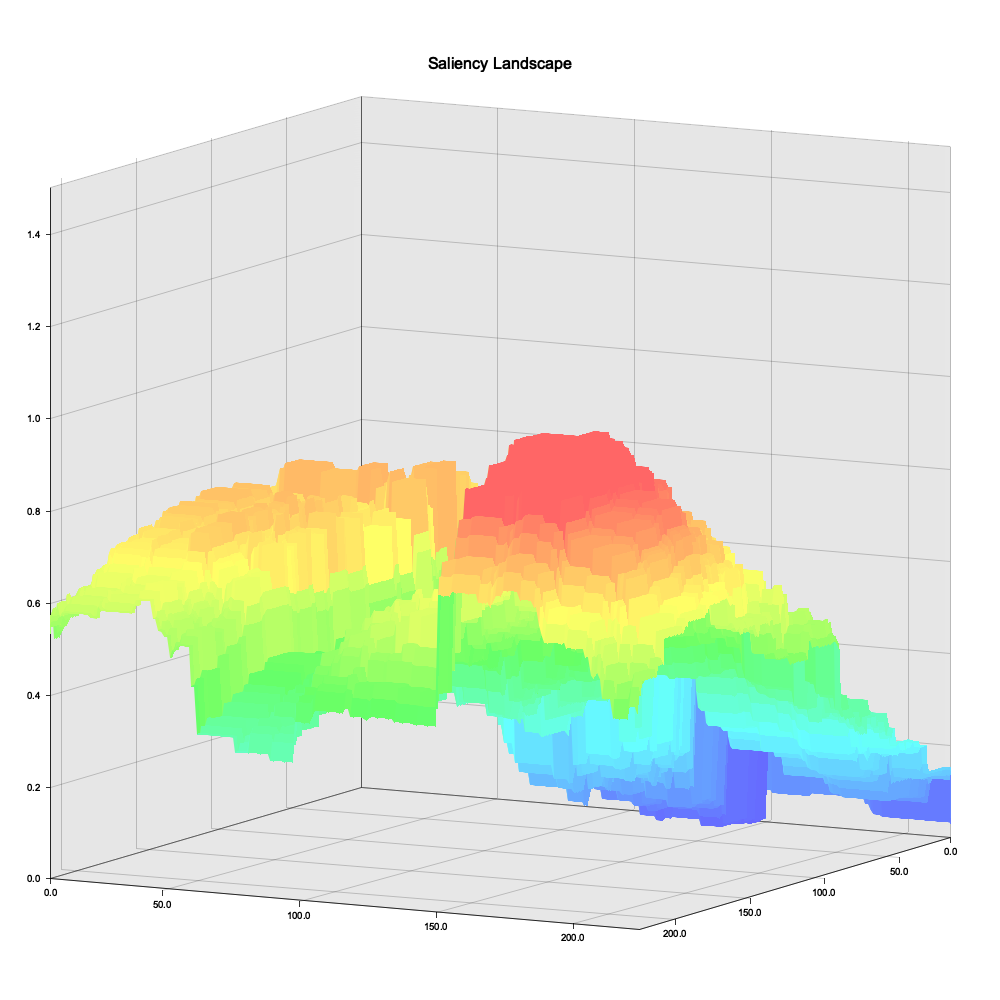}
        \caption{Responsibility landscape}
        \label{occ:landscape}
    \end{subfigure}
    \hfill
     \begin{subfigure}{0.2\textwidth}
        \centering
        \includegraphics[scale=0.3]{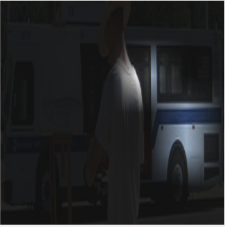}
        \caption{\sag explanation}
        \label{occ:sag}
    \end{subfigure}
    \vspace{5mm}
    \caption{\Cref{occ:one} shows a disjoint explanation produced by \multideepcover for \Cref{occ:bus}. \multideepcover also produces the explanation~\Cref{occ:two} which contains just a thin strip of the bus on the left hand side. While a smaller explanation, a human might prefer~\Cref{occ:one}. \Cref{occ:landscape} shows the responsibility landscape \multideepcover produces. All explanations are resized to $224\times224$ as per model requirements. \Cref{occ:sag} shows that \sag fails to find a disjoint explanation and included the man's shoulder.}
    \label{fig:occ}
\end{figure*}

As \multideepcover's explanation extraction algorithm is based on a spatial search, discovery of non-contiguous 
explanations can become challenging. We have not conducted a comprehensive study on multiple explanations of partially occluded images. 
However, a preliminary experiment indicates that \multideepcover successfully finds a non-contiguous explanation of a partially
occluded image, if supported by the model (\Cref{occ:sag}). \sag, in contrast,
failed to find a non-contiguous explanation. 


\end{document}